\documentclass{article}
\usepackage{spconf,amsmath,graphicx,hyperref}
\usepackage[utf8]{inputenc} 
\usepackage[T1]{fontenc}    
\usepackage{hyperref}       
\usepackage{url}            
\usepackage{booktabs}       
\usepackage{amsfonts}       
\usepackage{nicefrac}       
\usepackage{microtype}      
\usepackage{xcolor}         
\usepackage{lipsum}
\usepackage{amsmath, amssymb, amsthm}
\usepackage{algorithm}
\usepackage{algpseudocode}
\usepackage{thmtools}
\usepackage{thm-restate}
\usepackage{commands}
\usepackage{graphicx} 
\usepackage{cite}
\usepackage{enumitem}
\usepackage{adjustbox}

\let\OLDthebibliography\thebibliography
\renewcommand\thebibliography[1]{
  \OLDthebibliography{#1}
  \setlength{\parskip}{0pt}
  \setlength{\itemsep}{0pt plus 0.3ex}
}

\allowdisplaybreaks

\DeclareMathOperator*{\argmin}{arg\,min}
\newtheorem{prop}{Proposition}

\newtheorem{thm}{Theorem}
\newtheorem{lemma}[thm]{Lemma}

\newcommand*\Let[2]{\State #1 $\gets$ #2}
\algrenewcommand\algorithmicrequire{\textbf{Precondition:}}
\algrenewcommand\algorithmicensure{\textbf{Postcondition:}}

\title{Gromov-Wasserstein Graph Coarsening}
\name{Carlos A. Taveras, Santiago Segarra, and C{\'e}sar A. Uribe }
\address{Rice University, Houston, TX 77005, USA}

\begin{document}
%
\maketitle
\begin{abstract}
We study the problem of graph coarsening within the Gromov-Wasserstein geometry. 
Specifically, we propose two algorithms that leverage a novel representation of the distortion induced by merging pairs of nodes. 
The first method, termed Greedy Pair Coarsening (GPC), iteratively merges pairs of nodes that locally minimize a measure of distortion until the desired size is achieved. 
The second method, termed $k$-means Greedy Pair Coarsening (KGPC), leverages clustering based on pairwise distortion metrics to directly merge clusters of nodes. 
We provide conditions guaranteeing optimal coarsening for our methods and validate their performance on six large-scale datasets and a downstream clustering task. 
Results show that the proposed methods outperform existing approaches on a wide range of parameters and scenarios.
\end{abstract}
\section{Introduction} \label{sec:intro}
With the advent of Graph Neural Networks (GNNs) \cite{kipf2016semi}, there has been a surge in the development of data-driven algorithms for processing structured data, including graphs and their generalizations (e.g., simplicial complexes, cell complexes, and hypergraphs) \cite{xu2019GNNs, battiloro2025}.
These developments are part of a broader trend in data science of leveraging topological, geometric, and algebraic structures to process non-Euclidean data \cite{Petersen2019, bronstein2021geometric, Leus_2023}, and have been applied to solve problems in drug discovery \cite{wang2025drug}, social network analysis \cite{fan2019}, finance \cite{wang_2022}, and wireless communications \cite{olshevskyi2024fully}.
The performance of these models depends critically on the quantity and quality of the data used to train them. 
Consequently, model training can be both time- and resource-intensive, sometimes prohibitively so.

For graph-like data, there are three main dimensionality reduction paradigms: coarsening \cite{chen2022graph, Chen_2023, Jin_Loukas_JaJa, Bravo-Hermsdorff_Gunderson_2020}, sparsification \cite{Bravo-Hermsdorff_Gunderson_2020}, and condensation \cite{gao2025}, differing in how they consolidate graph structures.
Graph coarsening reduces the graph size by clustering nodes into partitions in a way that minimizes a chosen reconstruction error. Graph sparsification reduces the complexity of the graph by removing a subset of nodes and edges. 
Graph condensation learns small graphs that synthesize aspects of the original graph, including node features and topology.
\textit{We focus in this work on graph coarsening.}

Graph coarsening has a history dating back to at least \cite{kron1939tensor}, where it was used to reduce the complexity of electrical circuits. 
In scientific computing, it has been used to develop multigrid methods for solving differential equations \cite{chen2022graph}.
Moreover, graph coarsening algorithms have been employed for learning network node embeddings \cite{chen2017harp} and neighborhood pooling in GNNs \cite{ying2018}. 
For further elaboration on the history of graph coarsening, see ~\cite{chen2022graph}.
Central to the design of all coarsening algorithms, no matter the application, is the question of what notion of similarity to preserve between the original graph and its coarsened counterpart.
In graph data science, there has been an emphasis on preserving some notion of spectral similarity \cite{Jin_Loukas_JaJa, Bravo-Hermsdorff_Gunderson_2020}, typically with respect to a graph's Laplacian representation. 
While many interesting graph properties can be derived that are related to their spectra \cite{chung1997spectral}, there are some properties that cannot be (several graphs may have the same spectrum, i.e., \textit{cospectral graphs} \cite{VanDam_2003}). 

In recent years, a metric that has garnered much interest in graph data science and, more recently, in graph coarsening is the Gromov-Wasserstein (GW) distance. 
The GW  distance \cite{Chowdhury_Memoli_2019} is appealing for use in graph mining for several reasons.
First, the GW distance can be computed between unaligned graphs of different sizes, unlike the commonly used Frobenius norm distance or the Bures-Wasserstein distance \cite{bhatia2017}.
Second, the GW distance produces an alignment (transport plan) between nodes across graphs, finding utility in applications including graph matching and partitioning \cite{Xu_Luo_Zha_Carin_2019}.  
While solving for such an alignment is generally NP-hard, many algorithms have been proposed to efficiently approximate the alignment \cite{Xu_Luo_Carin_2019}.
Third, we can easily compute geodesics between pairs and barycenters between groups of networks. 
To this end, GW geodesic and barycenter-based methods have been used for applications including data augmentation \cite{zeng2024graph}, and graph clustering \cite{Xu_2020}.
Finally, the Gromov-Wasserstein distance induces an equivalence relationship between networks of different sizes, which can be used to reduce the size of a graph to its minimal representative without any loss in information.

GW-based graph coarsening has been previously considered using the signless Laplacian representation, where it was shown that the GW distance can be bounded by a spectral distance~\cite{Chen_2023}.
In this case, the coarsening problem is approximated using the weighted kernel $k$-means algorithm.
Moreover, \cite{Chen_2023} experimentally demonstrates that prioritizing spectrum preservation may not be ideal for tasks such as graph classification, lending merit to the utility of GW distance preservation.
We focus on developing graph coarsening algorithms that minimize GW distance, though we put no restrictions on the graph representation used.
Towards this goal, we propose two algorithms, Greedy Pair Coarsening (GPC) and $k$-means Greedy Pair Coarsening (KGPC), which exploit the similarity of a pair of nodes, as characterized by merging distortion, to coarsen graphs.
We summarize our contributions as follows:
\begin{enumerate}[ leftmargin=2.0em,left=-1pt,itemsep=1pt, parsep=-1pt, topsep=-0pt, partopsep=-1pt]
\item We propose an iterative graph coarsening method that, under appropriate assumptions, is guaranteed to recover the smallest representation of a measure network within its weak isomorphism class. 
\item We propose a novel network representation, based on the distortion induced by merging node pairs, for use in a $k$-means clustering method, providing a more efficient alternative to the iterative method.
\item We corroborate the utility of our algorithms by comparing the distortion induced by coarsening with that of several established algorithms and their performance on downstream tasks, such as graph classification.
\end{enumerate}

\section{Preliminaries} \label{sec:prelims}
\textbf{Graphs:} A (weighted) graph $G = (V, E, w)$ is a triplet consisting of a finite set of nodes $V = \{v_1, \cdots, v_n\}$, a set of edges $E \subseteq V \times V$, and a mapping $w: E \to \mathbb{R}$ of edges to real numbers. 
A graph is undirected if $w(v_i, v_j) = w(v_j, v_i)$ for all $(v_i, v_j) \in E$.
We can represent a graph by its adjacency matrix $A \in \R^{n \times n}$ where $a_{ij} = w(v_i, v_j)$ if $(v_i, v_j) \in E$; otherwise, $a_{ij} = 0$.
Note that if $G$ is undirected, its adjacency matrix is symmetric.
The adjacency matrix is but one choice of representation for graphs; undirected graphs without self-loops (i.e. $(v,v) \notin E$ for all $v \in V$) are often represented by their Laplacian matrix $L = D-A$, where $D = \text{diag}(A \one_n)$, and its variants, including the signless Laplacian $L' = D+A$ which is used in \cite{chen2022graph}.
We denote the choice of matrix representation for a graph by $S \in \mathbb{R}^{|V|\times |V|}$ and the weight of the edge from $v_i$ to $v_j$ by $s_{ij}$.

\textbf{Basics of Gromov-Wasserstein Pseudo-Metrics:} The Gromov-Wasserstein (GW) ``distance'' \cite{sturm2006geometry, Memoli_2011} is a pseudometric on the space of measure networks \cite{Chowdhury_Memoli_2019}.
A finite (measure) network $(X, W_X, \mu_X)$ is a triplet consisting of a finite set $X$ of nodes, a weight function $W_X : X \times X \to \R$, and a fully-supported probability measure $\mu_X$. 
Abusing notation, we also denote measure networks by $(S, \mu)$, where $S \in  \R^{|V|\times|V|}$ is a matrix for which $s_{ij} = S_X(x_i, x_j)$ and $\mu=\mu_X$.
We denote the set of all finite measure networks by $\mathcal{N}$ and of $N$-node measure networks by $\mathcal{N}_N$.
Hereafter, we use the terms measure network and network interchangeably.

A (measure) coupling $\pi$ of two networks $X$ and $Y$ is a probability measure on the product space $X \times Y$ satisfying the marginal constraints $\mu_Y(y) = \sum_{x_i} \pi(x_i, y)$ and $\mu_X(x) = \sum_{j} \pi(x, y_j)$.
We denote the set of all such couplings by $\Pi(\mu_X, \mu_Y)$.
The distortion of $X$ and $Y$ with respect to the measure coupling $\pi \in \Pi(\mu_X,\mu_Y)$ is defined by
\begin{align}\label{eq:distort}
\dis^2(\pi) = \langle \mathcal{L}^2_2(X,Y) \otimes \pi, \pi\rangle.
\end{align}
where $\mathcal{L}_2^2(X, Y) \otimes \pi$ is the tensor product defined by
\vspace{-1 mm}
\begin{align} \label{eq:distort_full}
[\mathcal{L}_2^2(X, Y) \otimes \pi]_{ik} = \sum_{jl} \lvert S_X(x_i, x_j)- S_Y(y_k, y_l)\rvert^2 \pi(x_j, y_l),
\end{align}
and $\langle \cdot, \cdot \rangle$ is the Frobenius inner product \cite{Peyre_Cuturi_Solomon_2016}.
The ($L^2$-) GW distance between networks $X$ and $Y$ is then
the distortion of the infimizing coupling
\begin{align}\label{eq:gw_dist}
\dgw(X, Y) = \inf_{\pi \in \Pi(\mu_X, \mu_Y)} \dis^2(\pi),
\end{align}
implying that $d_{GW}(X,Y) \leq \dis(\pi)$ for any $\pi \in \Pi(\mu_X,\mu_Y)$.
Equipped with the GW distance, 
$(\NN, d_{GW})$ is a pseudometric space \cite{Chowdhury_Memoli_2019}, i.e., it is a metric space up to weak isomorphism (see Appendix \ref{sec:weakiso}).
We use measure networks to model graphs; 
the graph $(V, E, S)$ is represented by the measure network $G=(V, S, \mu)$, where $\mu$ can be chosen to reflect the relative importance of nodes.
We set $\mu$ to the uniform measure over $V$ by default.
We denote the vector of weights emanating to/from a node $v \in V$ by 
$S(v) = \begin{bmatrix}S(v, v_1), \cdots, S(v, v_n)\end{bmatrix}.$

\textbf{Graph Coarsening:}
Graph coarsening partitions the node set $V$ into $M$ sets, where $|V|=N > M$.
We denote this mapping by $p: V \to \{1, \cdots, M\}$ and the set of nodes in the $j$-th partition set, or supernodes,  by $P_j = p^{-1}(j)$.
We can encode such a partition by an assignment matrix $C_p \in \{0,1\}^{N \times M}$ where 
\vspace{-1 cm}

\begin{align}    
C_p(i,j) = 
\begin{cases} 
1, & \text{if } v_i\in P_j\\
0, & \text{otherwise.} 
\end{cases}
\end{align}

\vspace{-0.3 cm}
We denote the space of all assignment matrices from $N$-node to $M$-node networks by $\C_{N,M}$.
Given an assignment matrix $C_p \in \C_{N,M}$, we can form the average coarsening matrix $\C_w$ defined as 
\begin{align*}
C_w = \diag(\mu) C_p\diag(1/\mu'),
\end{align*}
where $\mu' = C_p^\top \mu \in \R^M$ and $C_w \in \mathbb{R}^{N\times M}$.
Coarsened graphs are constructed using average coarsening matrices as follows:
\begin{align} \label{eq:bary}
S'=C_w^\top S C_w, \quad \text{ and  } \quad  \mu' = C_p^\top \mu.
\end{align}

\noindent The coarsened matrix representation $S'$ consolidates edge weights by taking a convex combination of the weights being merged with respect to the relative mass of the nodes being merged, see~\eqref{eq:coarsened_weight} for an explicit characterization.
Moreover, it was shown in \cite[Appendix B.3]{Chen_2023} that the transformation of $S$ defined in Eq. \eqref{eq:bary} is a semi-relaxed Gromov-Wasserstein barycenter \cite{vincentcuaz_2022} (see Appendix \ref{sec:gw-sketching}).
\vspace{-3 mm}

\section{Gromov-Wasserstein Graph Coarsening} \label{sec:method}
Given an $N$-node measure network $G = (S, \mu)$, the Gromov-Wasserstein coarsening problem seeks an assignment matrix $C_p^*$ that solves
\begin{align} \label{eq:coarsening_formulation}
C_p^* 
& = \argmin_{C_p \in \C_{N,M}} 
\lVert 
(S- C_p {C}_w^\top S {C}_w {C}_p^\top)\odot(\mu\mu^\top)^{1/2}
\rVert_F^2.
\end{align} 
Recall that coarsening reduces the size of a network by finding an optimal partition $\P = \{P_1, \cdots, P_M\}$ merging the nodes in these subsets to produce the coarsened network $G'$. 
If $C_p$ encodes the assignment of nodes to partition sets (supernodes), the matrix $\pi = \diag(\mu) C_p$ is a transport plan from $G$ to $G'$, for any $G' \in \N_M$ that is formed by merging nodes in $G$ (i.e., no mass splitting).
Given the assignment $C_p$, we can construct the measure coupling $\pi = \diag(\mu) C_p$ and it was shown in \cite[Appendix B.3]{Chen_2023} that
\begin{align*}    
G' &=  
\argmin_{G' \in \N_M} \langle \mathcal{L}_2^2(G, G') \otimes \diag(\mu) C_p, \diag(\mu) C_p\rangle \\ & =
(C_w^\top S C_w, C_p^\top \mu);
\end{align*}

\noindent in other words, $G'$ minimizes $\dis(\pi)$ for a fixed $C_p$. 
Therefore, minimizing the distortion over $C_p$ gives the best coarsening. 
Moreover, it can be shown that $d_{GW}(G', G^{''}) = 0$ when $G^{''} = (C_p C_w^\top S C_w C_p^\top, \mu)$.
Taken together, we get Eq.~\eqref{eq:coarsening_formulation};
this formulation is closely related to the Gromov-Wasserstein sketching problem \cite{memoli_sidiropoulos_singhal_2018} which, given $G \in \N_N$, seeks the network $G' \in \N_M$ closest to $G$, or $G' = \argmin_{G' \in \N_M} d_{GW}(G, G')$.
The connection between coarsening and sketching is expounded on in Appendix \ref{sec:gw-sketching}.

We propose two approaches to tackle problem \eqref{eq:coarsening_formulation}, both of which are based on the distortion induced by merging a pair of nodes. We leverage the intuition that nodes within the same partition should have similar neighborhoods. 
The first method greedily merges node pairs by choosing the coarsening from $N$ to $N-1$ nodes with minimal discrepancy.
The second method derives a graph representation from the distortion induced by merging node pairs, which partitions nodes using the $k$-means algorithm.

\subsection{Greedy Pair Coarsening (GPC)}
The first method we propose to solve Problem  \eqref{eq:coarsening_formulation} is Greedy Pair Coarsening (GPC).
First, note that, given a network $G = (S, \mu)$, we can reformulate Problem \eqref{eq:coarsening_formulation} as 
\vspace{-2 mm}
\begin{align} \label{eq:coarse_reformulation}
\min_{E_p^1, \cdots, E_p^{N-M}} 
\lVert 
(S - R^\top S R)\odot(\mu\mu^\top)^{1/2}
\rVert_F^2,
\end{align}
\vspace{-0.7 mm}
\noindent where $M$ is the desired coarsening level, $R = C_w C_p^\top$, 
\begin{align*}    
C_p & = E_p^{(1)} \cdots E_p^{(N-M)} 
&
C_w & = \diag(\mu) C_p \diag(1/C_p^\top \mu)
\end{align*} 
and where each $E_p^{(j)} \in \C_{N-j+1,N-j}$ corresponds to the merging of a pair of (super) nodes.

\noindent We can approximate a solution to Problem \eqref{eq:coarse_reformulation} by solving a sequence of greedy optimization problems to produce $(E_p^{(i)})_{i=1}^{N-M}$, each of which minimizes some intermediate cost.
In particular, given the first $i$-$1$ assignment matrices $E_p^{(1)}, ..., E_p^{(i-1)}$, we let $C_p^{(i-1)} = E_p^{(1)} \cdots E_p^{(i-1)}$, and solve 
\begin{align}
E_p^{(i)} &= \argmin_{E_p \in C_{N-i+1,N-i}} F(E_p) \label{eq:greedy_coarse}\\ 
F(E_p) & := 
\lVert (S - S^{(i)}(E_p)) \odot (\mu\mu^\top)^{1/2} \rVert_F^2 \notag \\
S^{(i)}(E_p) & =R^{(i)}(E_p)^\top S R^{(i)}(E_p) \notag
\\
R^{(i)}(E_p) & = C_w^{(i)}(E_p) (C_p^{(i)}(E_p))^\top \notag \\
C_p^{(i)}(E_p) &= C_p^{(i-1)} E_p \notag \\ 
C_w^{(i)}(E_p) &= \diag(\mu) C_p^{(i)}(E_p) \diag(1/(C_p^{(i)}(E_p))^\top \mu) \notag,
\end{align}
where $\mu \in \mathbb{R}_+^N$ and $\mu^\top 1_N = 1$.
Performing this optimization $N-M$ times, we obtain a feasible assignment matrix $C_p^{(N-M)} = E_p^{(1)} \cdots E_p^{(N-M)} \in \C_{N,M}$ for Problem~\eqref{eq:coarsening_formulation}.
Since $\mathcal{C}_{N-i+1,N-i}$ is a discrete set with $\binom{N-i+1}{2}$ elements, we can determine the minimizing transport plan for Eq. \eqref{eq:greedy_coarse} directly by comparing distortions.
\vspace{-0.25 mm}
Note that when the distortion induced by a node merging is zero, the coarsened network remains in the weak isomorphism class of the original network (see Proposition \ref{prop:weak_iso}).
Once the minimal representative is achieved, we must merge nodes that incur some distortion/error. 
This process is repeated until a network of the desired size is recovered.
As a first theoretical result, we show that the GPC algorithm produces the minimal representative of the weak isomorphism class of a measure network.
\begin{prop} \label{prop:weak_iso}
Given a measure network $G$, GPC recovers the smallest network weakly isomorphic to $G$.
Moreover, when $k$ is the size of the minimal representative, GPC($G, N-k$) solves Problem~\eqref{eq:coarsening_formulation}.
\end{prop}

\noindent For networks admitting a node partition $\P$ for which nodes in the same partition set (supernode) have sufficiently similar weights, and nodes in different supernodes have sufficiently different weights, GPC produces a network that recovers $\P$.

\begin{prop} \label{prop:gpc_convergence}
Let $G = (V, S, \mu)$ be a symmetric $N$-node network whose nodes can be partitioned into sets $\mathcal{P} = \{P_1, \cdots, P_M\}$ for which there exist $\epsilon > 0$, $\alpha > 4+4\sqrt{N^2/(N-1)}$, satisfying $\lVert S(u_1) - S(u_2)\rVert_\infty < \epsilon$ for all $u_1, u_2 \in P_i$ and $\inf_{u\in V} |s(u_1, u)-s(u_2, u)| \ge \alpha\epsilon$  for $u_1 \in P_i, u_2 \in P_j$ for $i \neq j$. 
Then, for $G'=\text{GPC}(G, N-M)=(V^{(N-M)}, S^{(N-M)}, \mu^{(N-M)})$, we have $V^{(N-M)} = \P$.
\end{prop}

\subsection{$k$-means  Greedy Pair Coarsening (KGPC)} 

While GPC is guaranteed to find the minimal representative of a measure network, its time complexity is $O(N^4)$, as we must compute pairwise distortions for each iteration.
To remedy this, we propose $k$-means Greedy Pair Coarsening (KGPC), which runs at $O(N^2 +TNM^2)$ where $T$ bounds the number of $k$-means iterations and $M$ is the desired coarsening size.
As with GPC, we characterize node similarity using the distortion induced by merging node pairs.
The goal then is to group nodes with similar induced node pair merging distortion within the same partition.
Towards this, we construct a matrix $H \in \mathbb{R}_+^{|V| \times |V|}$ where $H_{ij} = \dis(\pi^{ij})$, where $\pi^{ij}$ is the transport plan merging nodes $v_i$ and $v_j$.
Equipped with $H$ and assuming $\mu=\mathbf{1}_N/N$ we can solve for the assignment matrix $C_p^*$ by
\begin{align} \label{eq:kgpc_objective}
C_p^* & = 
\argmin_{C_p \in \mathcal{C}_{N,M}}
\lVert H - {C}_p {C}_w^\top H {C}_w {C}_p^\top\rVert_F^2.
\end{align}
We have observed this derived representation to be especially useful for finding node partitions using the $k$-means algorithm.
Moreover, while there are currently no theoretical guarantees that this method minimizes ~\eqref{eq:coarsening_formulation} in the general case, the algorithm's performance seems to indicate that the cost function, Eq. \eqref{eq:kgpc_objective}, can be upper-bounded by the first ordered differences comprising $H$.

\hfill
\begin{minipage}{0.46\textwidth}
\begin{algorithm}[H]
  \caption{$k$-means Greedy Pair Coarsening (KGPC)} \label{alg:spectral}
  \begin{algorithmic}[1]
    \Statex
    \Function{\text{KGPC}}{$(S, \mu),  k$}
    \Let{$H_{ij}$}{$\dis(\pi^{ij})$}
    \Let{$C_p$}{$k$-means($H$, $k$)}
    \Let{$C_w$}{$\diag(\mu)C_p\diag(1/(C_p^\top\mu))$}
    \Let{$(S', \mu')$}{$(C_w^\top S C_w, \ C_p^\top \mu)$}
    \State \Return{$G=(S', \mu')$}
    \EndFunction
  \end{algorithmic}
\end{algorithm}
\end{minipage}

\vspace{-5 mm}
\section{Numerical Analysis} \label{sec:exp}

In this section, we detail the experiments conducted to validate the utility of GPC and KGPC -- all experimental code can be found here: \href{https://github.com/ctaveras1999/graph-coarsening}{https://github.com/ctaveras1999/graph-coarsening}.
We contextualize the performance of our methods against the Multi-level Graph Coarsening (MGC) and Spectral Graph Coarsening (SGC) algorithms proposed in \cite{Jin_Loukas_JaJa} and the Kernel Graph Coarsening (KGC) algorithm proposed in \cite{Chen_2023}.
To our knowledge, \cite{Chen_2023} is the only other coarsening method with the explicit aim of minimizing a GW distance. 
Their method requires the use of the signless Laplacian representation.
We represent graphs by their adjacency matrix, though our method imposes no restriction on representation.

We perform two experiments: 1) we compare the reconstruction error of the different methods to quantify coarsening quality, and 2) we leverage the GW dictionary learning method proposed in \cite{Xu_2020} for graph classification.
Throughout these experiments, we use several well-established graph datasets, namely the IMDB-Binary \cite{Yanardag_2015}, Mutag \cite{Debnath_1991}, Proteins \cite{Borgwardt_2005}, 
Enzymes \cite{Borgwardt_2005, Schomburg_2004}, 
PTC-MR \cite{morris_2020}, and 
MSRC \cite{Neumann_2014} datasets.

\vspace{1mm}
\noindent
\textbf{Quantifying Reconstruction Error:} The goal of this experiment is to determine which of the aforementioned methods best preserves the structure of the adjacency matrix as measured by the coarsening-induced distortion.
For each graph in a given dataset, we coarsen at various levels (15\% and 85\% of all nodes in 5\% increments), then compute the distortion between the graphs and their coarsenings.
At each coarsening level, we compute the distortion induced by coarsening and average over all such distortions.
This average distortion is then treated as a measure of coarsening quality as a function of coarsening level.
The results of this experiment for the MSRC dataset are summarized in Figure~\ref{fig:exp1}. 
In it, we can see that for MSRC and Enzymes, our algorithms achieve the lowest, or nearly the lowest, distortion across coarsening levels.
As the coarsening level increases, the difference between the methods becomes less pronounced, leading to little meaningful difference between the methods. 
This suggests that the graphs cannot be well-approximated by graphs with $80\%$ or more of the nodes coarsened.

\vspace{-4mm}
\begin{figure}[H]
    \centering
    \includegraphics[scale=0.28]{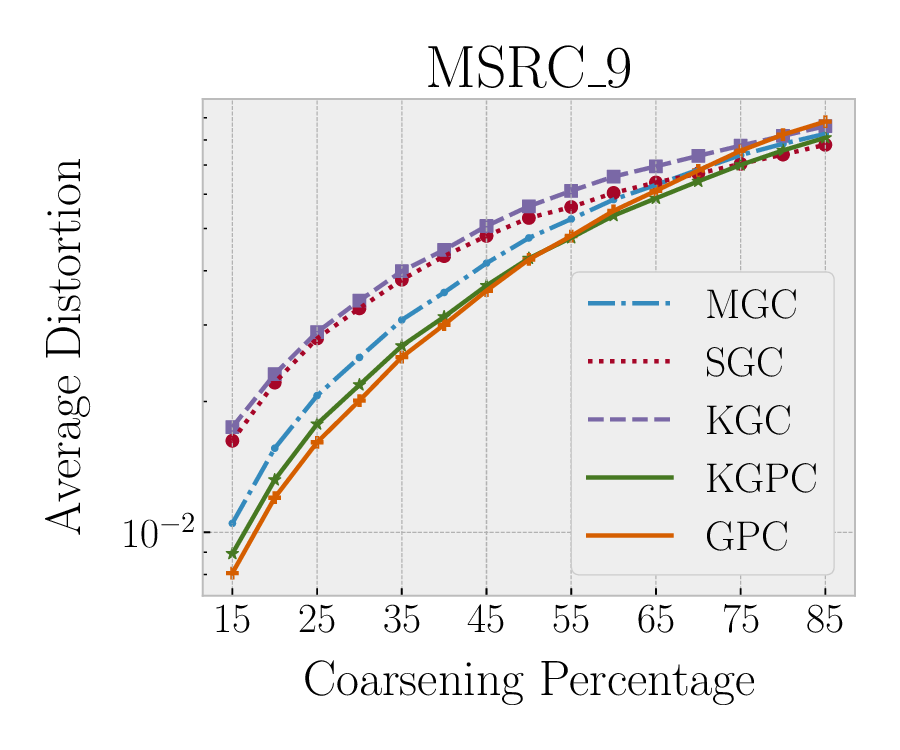}
    \hspace{-4 mm}
    \includegraphics[scale=0.28]{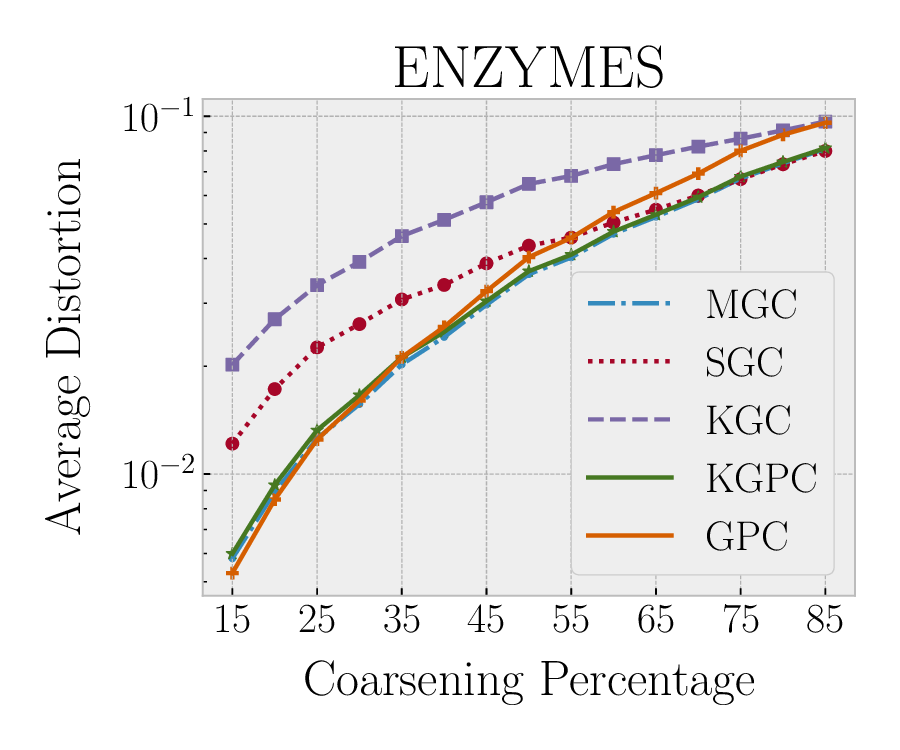}
    \vspace{-4mm}
    \includegraphics[scale=0.28]{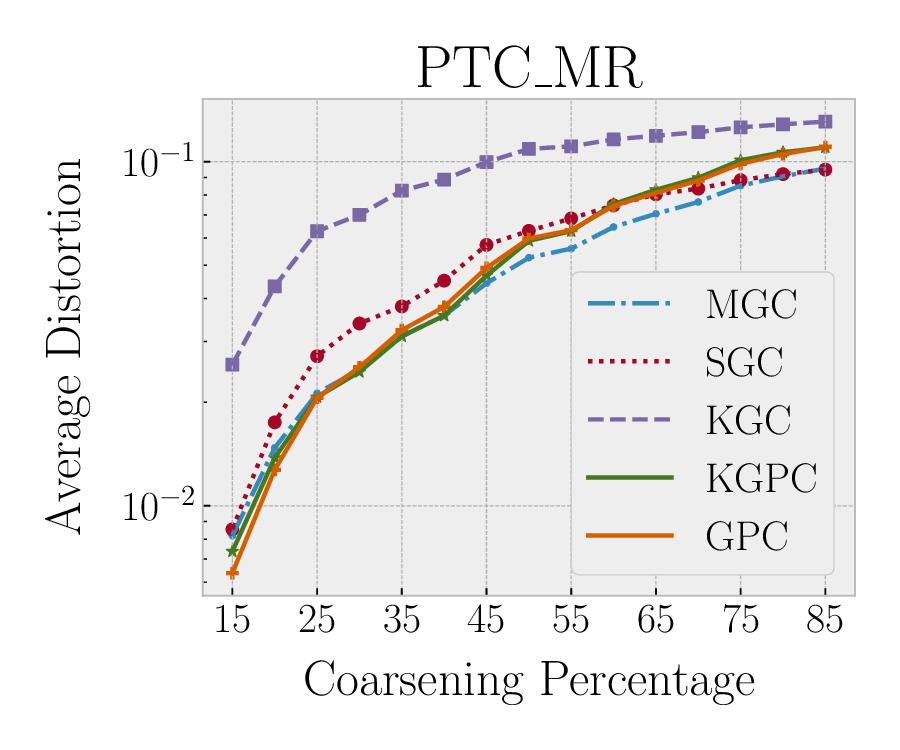}
    \hspace{-4 mm}
    \includegraphics[scale=0.28]{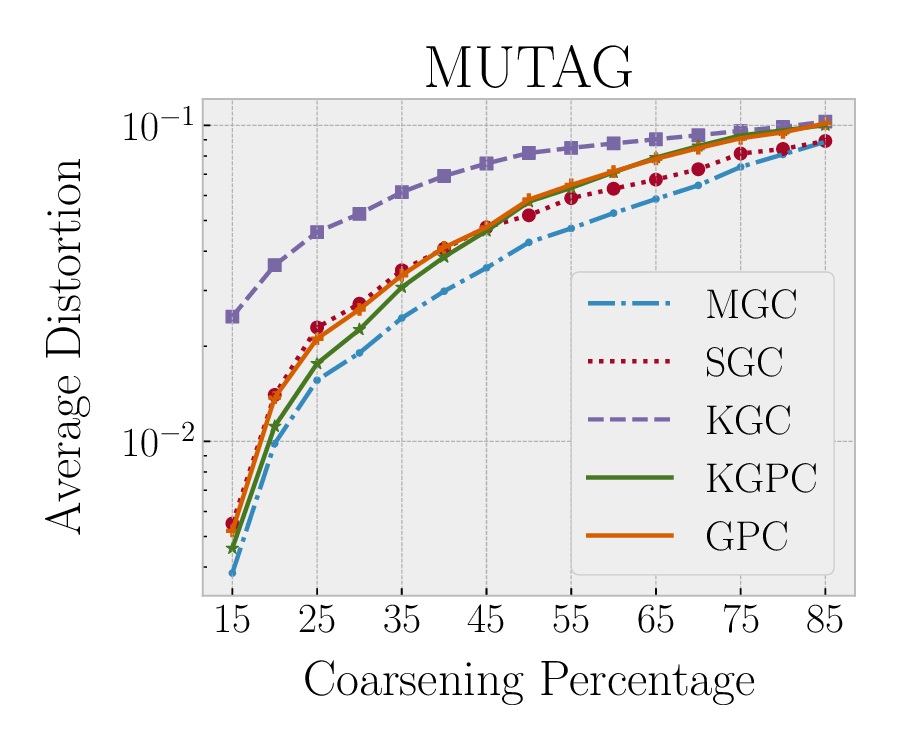}
    \caption{For each method and graph in a dataset, we coarsen between 15\% and 85\% of nodes, and average the distortion over all graphs. 
    GPC and KGPC achieve the overall lowest or near-lowest distortion on MSRC, Enzymes, and PTC-MR, whereas 
    MGC performs best on MUTAG.}
    \label{fig:exp1}
\end{figure}

\vspace{1mm}
\noindent
\textbf{Graph Classification via Clustering:} For this experiment, we leverage the Graph Dictionary Learning method (FGWF) proposed in \cite{Xu_2020} for unsupervised graph classification. 
Given a set of graphs, $\mathcal{G} = [G_i]_{i=1}^{N}$, the objective of this method is to learn a set of atoms $[B_j]_{j=1}^{M}$ and weights $[\lambda_{i}]_{i=1}^{N}$ such that barycenters formed by the dictionary atoms can well-approximate graphs in the dataset.
We initialize the dictionary with 15 graphs randomly sampled from the dataset and randomly initialize the weights $\lambda_{i}$ associated with each graph in the dictionary. 
Model parameters are updated using the Adam optimizer.
Individual FGWF models were trained for 15 epochs with a learning rate of $0.01$.
After training, we classify graphs by clustering their associated weight values using the $k$-means algorithm.

The goal of this experiment is to determine how faithfully the different coarsening methods represent the original data. 
To achieve this, we reduce all data to 40\% of its original size. For each combination of coarsening method and dataset, we train separate FGWF models and compute the Rand Index of the respective classification results.
Table \ref{table:exp2} reports the average and standard deviation of the Rand Index for each combination of dataset and algorithm, and indicates that GPC and KGPC are, overall, the most compatible with the FGWF model \cite{Xu_2020} for graph classification. 
The Rand index for our algorithms was on par with or better than the ones produced by the original graph. 
This may indicate that the coarsening step helps remove extraneous structure in the graph that does not aid in classification. 

\begin{table}
\hspace{-0.1 cm}
\adjustbox{max width=8.5 cm}{
\begin{tabular}{lrrrrrr}
Methods\textbackslash Datasets & \multicolumn{1}{l}{IMDB-B} & \multicolumn{1}{l}{MUTAG} & \multicolumn{1}{l}{Proteins} & \multicolumn{1}{l}{MSRC-9} 
& \multicolumn{1}{l}{PTC-MR} \\ \hline
MGC \cite{Jin_Loukas_JaJa} & 
${50.7{\pm}0.2}$ & 
$50.1 {\pm} 1.4$  & 
$54.3{\pm}3.4$ & 
$\mathbf{78.0 {\pm} 0.3}$ & 
$50.7 {\pm} 0.2$ \\
SGC \cite{Jin_Loukas_JaJa} & 
$50.5 {\pm} 0.2$ & 
$50.3 {\pm} 0.6$ & 
52.8 $\pm$ 3.6 & 
77.9 $\pm$ 0.3 & 
50.5 $\pm$ 0.2 \\
KGC \cite{Chen_2023} & 
$50.0 {\pm} 0.0$ & 
$51.1 {\pm} 1.4$ & 
$\mathbf{58.1 {\pm} 2.0}$ & 
$77.9 {\pm} 0.2$ & 
$50.0 {\pm} 0.0$ \\ 
\hline
GPC (Ours)  & 
$\mathbf{51.3{\pm}0.3}$
& $51.0 {\pm} 1.6$
& $55.7 {\pm} 2.3$
& $77.9 {\pm} 0.4$
& $\mathbf{51.3 {\pm} 0.2}$ \\
KGPC (Ours) & 
$50.8 {\pm} 0.3$ & 
\textbf{53.6 $\pm$ 2.5}      
& 56.7 $\pm$ 2.1                     
& $\mathbf{78.0 {\pm} 0.4}$      
& 50.8 $\pm$ 0.3  \\ 
\hline
Original    
& $51.1 {\pm} 0.8$ & 
$50.6 {\pm} 0.6$ & 
$54.9 {\pm} 2.3$ & 
$77.9 {\pm} 0.2$ & 
$51.1 {\pm} 0.8$                 
\end{tabular}
\\
\vspace{-8mm}}
\caption{ \label{table:exp2} Rand Index for coarsened graphs with $60$\% of nodes coarsened. For each combination of dataset and method, we coarsen the data, which we then use to train four FGWF models~\cite{Xu_2020} for graph classification.
We use the Rand Index to measure the quality of the classification results and find that for most datasets, either GPC or KGPC performs best.\vspace{-4mm}}
\end{table}

\vspace{-4 mm}
\section{Discussion} \label{sec:outro}
We proposed the GPC and KGPC algorithms for graph coarsening with respect to the Gromov-Wasserstein distance.
We conduct two experiments, where results show that our methods generally produce coarsening with less distortion and better discriminability for classification tasks.
These results provide evidence that the proposed methods better leverage the structure granted by the GW geometry than others.
Moreover, these experimental results showcase the viability of graph coarsening in practical applications.

Future research should explore establishing an upper bound on the coarsening objective in terms of the distortion of first-order coarsening to supplement the development of the KGPC algorithm.
Other directions that this work opens up include incorporating graph features into the coarsening and better characterizing classes of graphs that can benefit most from coarsening.

\small
\section{Acknowledgments}
Part of this work is funded by the National Science Foundation under Grants \#2213568 and \#2443064. Research was partially sponsored by the Army Research Office and was accomplished under Grant Number W911NF-17-S-0002. The views and conclusions contained in this document are those of the authors and should not be interpreted as representing the official policies, either expressed or implied, of the Army Research Office or the U.S. Army or the U.S. Government. The U.S. Government is authorized to reproduce and distribute reprints for Government purposes notwithstanding any copyright notation herein.
\bibliographystyle{ieeetr}

\newpage
\appendix
\section{Proofs Omitted from the Main Text}
\subsection{Proof of Proposition 1}
If there exists a pair of nodes for which the induced distortion is equal to zero, then $G'$, the graph resulting from merging said nodes, is weakly isomorphic to $G$. 
This follows from the fact that the distortion is an upper-bound for the GW distance (GW evaluates the distortion of infimizing coupling), and the GW distance is unique up to weak isomorphism \cite{Chowdhury_Memoli_2019}.
We repeat this process until there are no node pairs with zero distortion, in which case we have reached a minimal representative.
Since the distortion between the original network and a minimal representative is zero, and the distortion is an upper-bound on the GW distance, it follows that the minimal representative solves Problem \eqref{eq:coarsening_formulation}.
Note that minimal representatives may not be unique, but all minimal representatives are strongly isomorphic (i.e., unique up to node re-labeling) \cite{Chowdhury_Memoli_2023}.

\subsection{Proof of Proposition 2}
Let $V^{(t)} = \{v_{1}^{(t)}, \cdots, v_{N-t}^{(t)} \}$ denote the partition, or supernode set, constructed after $t$ iterations of GPC. 
The $i$-th supernode $v_i^{(t)} \in V^{(t)}$ contains a set of nodes and we denote its size by $N_i^{(t)} = |v_{i}^{(t)}|$. 
We call a pair of supernodes $v_i^{(t)}, v_{j}^{(t)} \in V^{(t)}$ consistent if $v_i^{(t)} \cup v_{j}^{(t)} \subseteq P$ for some $P \in \mathcal{P}$. 
We proceed by induction towards showing that GPC only merges consistent pairs of supernodes for $1 \leq t < N-M$. 
This will then imply that after $N-M$ iterations we achieve $V^{(N-M)} = \P$.

Before the first iteration, the partition $V^{(0)}$ consists of singleton supernodes $v_i^{(0)} = \{v_i\}$. 
The first iteration of GPC thus merges the pair of nodes $v_{i}, v_{j}$, inducing the least distortion. 
By hypothesis and Lemma \ref{lemma:2}, the distortion induced by the merging of nodes in the same partition set is upper-bounded by $\epsilon$, and lower-bounded by $\alpha \epsilon$ for nodes in different partition sets.
Therefore, GPC must merge consistent nodes at the first iteration. 

Suppose now that GPC has run for $1 \leq t < N-M-1$ iterations during which only consistent supernodes were merged, to produce the supernode set $V^{(t)} = \{v_{1}^{(t)}, \cdots, v_{N-t}^{(t)}\}$.
Let $v_i^{(t)}, v_j^{(t)}, v_k^{(t)} \in V^{(t)}$ be such that $v_i^{(t)}$ and $v_j^{(t)}$ are consistent, and $v_i^{(t)}$ and $v_k^{(t)}$ are not. We want to show that the distortion induced by merging the former pair is strictly less than that of the latter. 
Lemma \ref{lemma:2} allows us to compute upper and lower bounds on  $D_{ij} = \dis(\pi^{ij})$ and $D_{ik} = \dis(\pi^{ik})$, respectively, where $\pi^{rs}$ is the transport map merging supernodes $v_r^{(t)}$ and $v_s^{(t)}.$  
By Lemma \ref{lemma:3}, we have that $\max D_{ij} < \min D_{ik}$, as long as $\alpha > (4+4N/\sqrt{N-1})$.  
Therefore, at iteration $t$, GPC will merge the pair of consistent nodes with the least distortion. 

After $N-M$ iterations, no pair of supernodes in $V^{(N-M)}$ is consistent. 
This implies that the supernodes in $V^{(N-M)}$ correspond exactly with the sets in the partition $\P$, thus $V^{(N-M)} = \P$, as desired.

\begin{lemma}  \label{lemma:1}
Let $G = (V, w, \mu)$ be a symmetric measure network and $\pi^{12}$ the transport map induced by merging nodes $v_1, v_2 \in V$. 
Then, $\dis^2(\pi^{12}) = \mu_1\mu_2(A_{1} + 2A_{2})/(\mu_1+\mu_2)^4$ where $\Delta_{klmn} = (s_{kl} - s_{mn})$ and 
\begin{align*}
A_{1} & = \mu_1^{3}\mu_2(4\Delta_{1112}^{2}+\Delta_{2211}^2) + \mu_1\mu_2^3(\Delta_{1122}^2 +4\Delta^2_{2212}) \\
& +2(\Delta_{1211}^2\mu_1^4+\Delta_{1222}^2\mu_2^{4}) +4\mu_1^2\mu_2^2 [(|\Delta_{1112}| - |\Delta_{2212}|)^2 \\ &+ 2|\Delta_{1112}||\Delta_{2212} | + \Delta_{1112}\Delta_{2212}]\\
A_{2} & = (\mu_1+\mu_2)^3 \sum_{n=3}^{N} \mu_n \Delta_{1n2n}^2.
\end{align*}
\end{lemma}
\begin{proof}
We can represent the assignment matrix merging nodes $v_1$ and $v_2$ by 
$$
C^{12}_{p} = 
\begin{bmatrix} 
1_{2 \times 1} & 0_{2 \times (N-2)} \\ 
0_{(N-2) \times 1} & I_{N-2}
\end{bmatrix},
$$
where $1_{n_1\times n_2}$ (resp. $0_{n_1\times n_2}$) is the ones (resp. zeros) matrix $n_1$ rows and $n_2$ columns and $I_n$ is the identity matrix with $n$ rows and columns. 
Then, letting $\pi = \pi^{12}$ and $C_p = C^{12}_p$, we get
\begin{align*}    
\pi &= \diag(\mu) C_p, \\
C_w &= \pi \diag \left(1/C_p^\top \mu \right) \\
G' &= (V', C_w^\top S C_w, C_p^\top\mu).
\end{align*}
Carrying out the multiplications, we get 
\begin{align} \label{eq:coarsened_weight}
S'
& =
\begin{bmatrix}
\sum_{i,j=1}^{2} \theta_i \theta_j s_{ij} & \sum_{i=1}^{2} \theta_i s_{i3} & \cdots & \sum_{i=1}^{2} \theta_i s_{iN} \\
\sum_{j=1}^{2} \theta_i s_{3j} & s_{33} & \cdots & s_{3N}\\
\vdots & \vdots & \ddots & \vdots \\
\sum_{j=1}^{2} \theta_j s_{Nj} & s_{N3} & \cdots & s_{NN}
\end{bmatrix}
\end{align}
where $\theta_1 = \mu_1/(\mu_1 + \mu_2)$ and $\theta_2 = \mu_2/(\mu_1 + \mu_2)$. 
We now start computing $\dis^2(\pi)$ of the coupling.

Let $d_{ijkl} = |s_{ij} - s'_{kl}|^2 \pi_{ik} \pi_{jl}$. 
Then, 
\begin{align*}
\dis^2(\pi) &= 
\sum_{i,j=1}^{N} \sum_{k,l=1}^{N-1} d_{ijkl} \\
& =
\sum_{i,j=1}^{N} \left[ 
d_{ij11} + 
\sum_{k=2}^{N-1} d_{ijk1} + 
\sum_{l=2}^{N-1} d_{ij1l} + 
\sum_{k,l=2}^{N-1}d_{ijkl}
\right]
\end{align*}

Let 
\begin{align}
D_1 &= \sum_{i,j=1}^{N} d_{ij11} &
D_2 &= \sum_{i,j=1}^{N}\sum_{k=2}^{N-1} d_{ijk1} \\
D_3 &= \sum_{i,j=1}^{N}\sum_{l=2}^{N-1} d_{ij1l} &
D_4 &= \sum_{i,j=1}^{N}\sum_{k,l=2}^{N-1}d_{ijkl}
\end{align}
\noindent
and $\Delta_{ijkl} = (s_{ij} - s_{kl})$. 
We proceed with a term-by-term expansion
\begin{align*}
D_1 
& = \sum_{i,j=1}^{2} d_{ij11} = D_{11} + D_{12} + D_{13}+ D_{14} \\
D_{11} & = 
|s_{11} - s'_{11}|^2 \pi_{11}\pi_{11} = 
|s_{11} - s'_{11}|^2 \mu_1^2 \\
D_{12}&=
|s_{12} - s'_{11}|^2 \pi_{11} \pi_{21} =
|s_{12} - s'_{11}|^2 \mu_1 \mu_2 \\
D_{13} & = 
|s_{21} - s'_{11}|^2 \pi_{21} \pi_{11} = |s_{21} - s'_{11}| \mu_1\mu_2 \\
D_{14} & = 
|s_{22} - s'_{11}|^2 \pi_{21}\mu_{21} = |s_{22} - s'_{11}|^2 \mu_2^2
\end{align*}

note that by the symmetry of $s$ we have $D_{12} = D_{13}$.
Then, 
\begin{align*}
D_{11} & = 
\mu_1^2|s_{11} - s'_{11}|^2 \\
& = 
\mu_1^2\left|\sum_{i,j=1}^{2} (s_{11} - s_{ij}) \theta_i \theta_j\right|^2\\
&= \mu_1^2|2\Delta_{1112}\theta_1\theta_2+\Delta_{1122}\theta_2^2|^2 \\
& = \mu_1^2\theta_2^2 (4 \Delta_{1112}^2\theta_1^2 + \Delta_{1122}^2\theta_2^2 + 4 \Delta_{1112}\Delta_{1122}\theta_1\theta_2)\\ \\
& = 
\frac{\mu_1^2 \mu_2^2}{(\mu_1+\mu_2)^4}
\left(
4 \Delta_{1112}^2\mu_1^2 + \Delta_{1122}^2\mu_2^2 + 4 \Delta_{1112}\Delta_{1122}\mu_1\mu_2
\right).
\end{align*}
Similar computations yield
\begin{align*}
D_{12}&=D_{13} \\
& = 
\frac{\mu_1\mu_2}{(\mu_1+\mu_2)^4}(\Delta_{1211}^2\mu_1^4 + \Delta_{1222}^2\mu_{2}^{4} + 2\Delta_{1211}\Delta_{1222}\mu_1^2\mu_2^2)\\
D_{14}
& = 
\frac{\mu_1^2\mu_2^2}{(\mu_1+\mu_2)^4}(\Delta_{1122}^2 \mu_1^2 + 
4\Delta_{1222}\mu_2^2 +
4 \Delta_{1122}\Delta_{1222} \mu_1\mu_2)\\
\end{align*}
Combining the above terms, and letting  $\hat{\mu}_{12}= \mu_1\mu_2/(\mu_1+\mu_2)^4$, we get
\begin{align*}
D_{1}/\hat{\mu}_{12}&= 
\mu_1^3\mu_2
(4 \Delta_{1112}^2+\Delta_{1122}^2)+
\mu_1\mu_2^3
(\Delta_{1112}^2+\Delta_{1222}^2) \\ &+ 2(\Delta_{1211}^2\mu_1^4 + \Delta_{1222}^2\mu_2^4)  \\
& + 4\mu_1^2\mu_2^2 
(\Delta_{1112}\Delta_{1122}+\Delta_{1211}\Delta_{1222}+\Delta_{1122}\Delta_{1222}) \\
& = 
\mu_1^3 \mu_2
(4 \Delta_{1112}^2+\Delta_{1122}^2)+
\mu_1\mu_2^3
(\Delta_{1112}^2+\Delta_{1222}^2) \\
& + 
2(\Delta_{1211}^2\mu_1^4 + \Delta_{1222}^2\mu_2^4) \\
& + 4\mu_1^2\mu_2^2 [(|\Delta_{1112}| - |\Delta_{1222}|)^2] \\ 
& + 4\mu_1^2\mu_2^2 (2|\Delta_{1112}||\Delta_{1222}| + \Delta_{1112}\Delta_{1222}])
\end{align*}
We proceed with $D_2$, noting that $D_3 = D_2$ by the symmetry of $w$,
\begin{align*}
D_2 
& = 
\sum_{i,j=1}^{N}\sum_{k=2}^{N-1} d_{ijk1} =
\sum_{i,j=1}^{N} \sum_{k=2}^{N-1} |s_{ij}-s'_{k1}|^2 \pi_{ik} \pi_{j1} \\
& = 
\sum_{i=3}^{N} \mu_i 
\left(
\left|s_{i1}-\sum_{l=1}^{2}\theta_l s_{il}\right|^2 \mu_1 + 
\left|s_{i2}-\sum_{l=1}^{2}\theta_l s_{l2}\right|^2  \mu_2
\right) \\
& = 
\sum_{i=3}^{N} \mu_i 
\left( 
|\theta_2 (s_{i1}-s_{i2})|^2 \mu_1 + |\theta_1 (s_{i2}-s_{i1})|^2 \mu_2
\right) \\
& = 
\sum_{i=3}^{N} \mu_i 
\left( 
|s_{i1}-s_{i2}|^2 \mu_1\theta_2^2 + |s_{i2}-s_{i1}|^2 \theta_1^2\mu_2
\right) \\
& = 
\frac{1}{(\mu_1+\mu_2)^2}\sum_{i=3}^{N} \mu_i 
\left( 
|s_{i1}-s_{i2}|^2 \mu_1\mu_2^2 + |s_{i2}-s_{i1}|^2 \mu_1^2\mu_2
\right) \\
& = 
\frac{\mu_1\mu_2}{(\mu_1+\mu_2)^2}\sum_{i=3}^{N} \mu_i 
\left( 
|s_{i1}-s_{i2}|^2 \mu_2 + |s_{i2}-s_{i1}|^2 \mu_1
\right) \\
& = 
\frac{\mu_1\mu_2}{(\mu_1+\mu_2)^2}\sum_{i=3}^{N} \mu_i 
|s_{i1}-s_{i2}|^2(\mu_1+\mu_2)\\
& = 
\frac{\mu_1\mu_2}{\mu_1+\mu_2} \sum_{i=3}^{N} \mu_i 
|s_{i1}-s_{i2}|^2\\
& = 
\frac{\mu_1 \mu_2}{\mu_1+\mu_2} \sum_{i=3}^{N} \mu_i 
|\Delta_{i1i2}|^2.
\end{align*} 
Finally, we compute $D_4$,
\begin{align*}
D_4 
& = \sum_{i,j=1}^{N}\sum_{k,l=2}^{N-1}d_{ijkl} \\
& = \sum_{i,j=3}^{N}\sum_{k,l=2}^{N-1}|s_{ij} - s'_{kl}|^2 \pi_{ik} \pi_{jl} \\
& = \sum_{i,j=3}^{N} |s_{ij} - s_{ij}| \mu_{i} \mu_j = 0
\end{align*}
Combining these terms, we get
\begin{align*}
D/\hat{\mu}_{12} 
& = 
2(\mu_1+\mu_2)^3 \sum_{i=3}^{N} \mu_i |\Delta_{i1i2}|^2 \\&+
\mu_1^3\mu_2
(4\Delta_{1112}^2+\Delta_{1122}^2)+
\mu_1\mu_2^3
(\Delta_{1112}^2+\Delta_{1222}^2) \\&+ 
2(\Delta_{1211}^2\mu_1^4 + \Delta_{1222}^2\mu_2^4) \\&
+ 4\mu_1^2\mu_2^2 (|\Delta_{1112}| - |\Delta_{1222}|)^2 \\& 
+ 4\mu_1^2\mu_2^2[2|\Delta_{1112}||\Delta_{1222}| + \Delta_{1112}\Delta_{1222}]
\end{align*}
\end{proof}

\begin{lemma} \label{lemma:2}
Let $G = (V, S, \mu)$ satisfy the hypotheses in Proposition \ref{prop:gpc_convergence}, with partition $\P = \{P_1, \cdots, P_{M}\}$, $\epsilon > 0$, and $\alpha > 4 + 4\sqrt{N^2/(N-1)}$. 
Let $G'$ be the coarsening of $G$ induced by $C_p$, and $\phi: V \to V'$ the mapping corresponding to $C_p$. 
Suppose there exist nodes $v'_i, v'_j, v'_k \in V'$ satisfying $\lVert S(u_1) - S(u_2) \rVert_\infty < \epsilon$ for all $u_1, u_2 \in \phi^{-1}(v'_i) \cup \phi^{-1}(v'_j)$ and $\inf_{z \in V} |s(u_1, z)-s(u_3,z)| > \alpha \epsilon$ for $u_1 \in \phi^{-1}(v'_i)$ and $u_3 \in \phi^{-1}(v'_k)$.
Then,
$$
\dis^2(\pi^{ij}) < 
\frac{32 \epsilon^2 \mu'_1 \mu'_2}{\mu'_1 + \mu'_2}
\text{ and }
\dis^2(\pi^{ik}) \ge 
\frac{\epsilon^2(\alpha-4)^2 \mu'_1\mu'_3}{2(\mu'_1+\mu'_3)}.
$$
\end{lemma}

\begin{proof}    
Without loss of generality, we let $v_i' = v'_1$ and $v_j' = v'_2$.
To prove this lemma, we must compute upper and lower bounds on several terms of the form $|\Delta'_{ijkl}|$ where $\Delta'_{ijkl} = s'_{ij} - s'_{kl}$. 
Note that by the symmetry of $S'$ we have 
$|\Delta'_{ijkl}| = 
|\Delta'_{jikl}| = 
|\Delta'_{ijlk}| = 
|\Delta'_{jilk}|.$

We first compute upper bounds on weight differences of $G$ and $G'$, which we use in later computations.
Without loss of generality we order the nodes in $V$ such that $\phi^{-1}(v'_l) = \left[v_{\hat{N}_{l-1}+1}, \cdots, v_{\hat{N}_{l-1}+N_{l}}\right]$,
where $N_0 = 0$, $N_l = \left|\phi^{-1}(v'_l)\right|$, and $\hat{N}_{l} = \sum_{r=0}^{l} N_{r}$. 
Moreover, given a supernode $v'_l$, we define $[\theta_m]_{m=1}^{N_l}$ where $\theta_m = \mu\left(v_{\hat{N}_l+m}\right)/\mu(v'_l)$. 
Finally, we abuse notation below, using $m \in \phi^{-1}(v'_l)$ in place of $v_m \in \phi^{-1}(v'_l)$.

\begin{align*}
|s'_{11} - s_{11}| & = 
\left| 
\sum_{l,m\in\phi^{-1}(v'_1)} 
\theta_l \theta_m (s_{lm} - s_{11})
\right| \\ 
& \leq 
\left| 
\sum_{l,m\in\phi^{-1}(v'_1)} 
\theta_l \theta_m |s_{lm} - s_{11}|
\right| \\
& < 2\epsilon
\left| 
\sum_{l,m\in\phi^{-1}(v'_1)} 
\theta_l \theta_m
\right| = 2\epsilon \\
|s'_{12} - s_{11}| & = 
\left| 
\sum_{l\in\phi^{-1}(v'_1)}
\sum_{m\in\phi^{-1}(v'_2)}
\theta_l \theta_m (s_{kl} - s_{11})
\right| \\ 
& \leq 
\left| 
\sum_{l\in\phi^{-1}(v'_1)}
\sum_{m\in\phi^{-1}(v'_2)}
\theta_l \theta_m |s_{lm} - s_{11}|
\right| \\
& < 
2\epsilon
\left| 
\sum_{l\in\phi^{-1}(v'_1)}
\sum_{m\in\phi^{-1}(v'_2)}
\theta_l \theta_m
\right| 
= 
2\epsilon \\
|s_{22}'-s_{11}| & = 
\left| 
s'_{22} - s_{1,N_1+1} + s_{1,N_1+1} - s_{11}
\right| \\
& \leq 
|s'_{22} - s_{1,N_1+1}| + |s_{1,N_1+1} - s_{11}| \\
& < 2 \epsilon + 2 \epsilon = 4 \epsilon \\ 
|s'_{1n} - s'_{2n}| & = 
|s'_{1n} - s_{1,\hat{N}_{n-1}+1} + s_{1,\hat{N}_{n-1}+1} - s'_{2n}| \\
& \leq 
|s'_{1n} - s_{1,\hat{N}_{n-1}+1}| + 
|s_{\hat{N}_{n-1}+1,\hat{N}_{n-1}+1} - s'_{2n}| \\
& \leq 
2\epsilon + 2 \epsilon \\
& < 4 \epsilon
\end{align*}
We can now produce bounds for the relevant terms in the distortion
\begin{align*}
|\Delta'_{1112}| 
& = 
|s'_{11}-s'_{12}| \\
& = 
|s'_{11}-s_{11}+s_{11}-s'_{12}| \\
& \leq 
|s_{11}' - s_{11}| + |s'_{12} - s_{11}| \\
& < 2\epsilon + 2 \epsilon = 4 \epsilon \\ \\
|\Delta'_{1222}| 
& = 
|s'_{12}-s'_{22}| \\
& = 
|s'_{12}-s_{\hat{N}_{1}+1,\hat{N}_{1}+1}+s_{\hat{N}_{1}+1,\hat{N}_{1}+1}-s'_{22}| \\
& \leq 
|s'_{12} - s_{\hat{N}_{1}+1,\hat{N}_{1}+1}| + |s'_{22} - s_{\hat{N}_{1}+1,\hat{N}_{1}+1}| \\
& < 
2\epsilon + 2 \epsilon = 4 \epsilon \\ \\
|\Delta'_{1122}| 
& = 
|s'_{11} - s'_{22}| \\
& \leq 
|s'_{11} - s'_{12}| + |s'_{12} - s'_{22}| \\
& = 
|\Delta'_{1112}| + |\Delta'_{1222}| \\
& < 
4\epsilon + 4\epsilon = 8\epsilon \\ \\ 
|\Delta'_{1n2n}| & = |s'_{1n} - s'_{2n}| \\
& < 4 \epsilon
\\ \\
||\Delta'_{1112}| - |\Delta'_{1222}|| 
& \leq 
|\Delta'_{1112}-\Delta'_{1222}| \\
& = 
|s'_{11}-s'_{12}+s'_{12} - s'_{22}| \\ 
&= 
|\Delta'_{1122}| \\
& < 8 \epsilon 
\end{align*}
\begin{align*}
2|\Delta'_{1112}||\Delta'_{1222}| + \Delta'_{1112}\Delta'_{1222} & \leq 3|\Delta'_{1112}||\Delta'_{1222}| \\
& < 48 \epsilon^2
\end{align*}

In summary 
\begin{align*}
|\Delta'_{1112}|, |\Delta'_{1222}|, |\Delta'_{2212}|, |\Delta'_{1n2n}|
&< 
4\epsilon \\
|\Delta'_{1122}|, |\Delta'_{2211}|, 
\left||\Delta'_{1112}|-|\Delta'_{2212}|\right| 
&< 
8 \epsilon \\
2|\Delta'_{1112}||\Delta'_{2212}|+\Delta'_{1112}\Delta'_{2212} &< 48\epsilon^2
\end{align*}

We now proceed with computing lower bounds.
Without loss of generality, we now let $v_i' = v'_1$, $v_l' = v'_2$.
\begin{align*}
\alpha \epsilon 
&\le 
|s_{11}-s_{1,N_{1}+1}| \\
&=
|s_{11} - s'_{11} + s'_{11} - s'_{12} + s'_{12} -s_{1,N_1+1}| \\
& \leq |s_{11}-s'_{11}| + |s'_{11}-s'_{12}|+ |s'_{12}-s_{1,N_1+1}| \\
& < 2\epsilon + |s'_{11}-s'_{12}| + 2\epsilon \\ 
\implies |\Delta'_{1112}| 
& \ge (\alpha-4)\epsilon \\ \\
\alpha \epsilon &<
|s_{1,\hat{N}_n+1}-s_{N_1+1,\hat{N}_n+1}| \\
& = |s_{1,\hat{N}_n+1}-s'_{1n}+s'_{1n}-s'_{2n} \\ & +s'_{2n}-s_{N_1+1,\hat{N}_n+1}| \\
& \leq
|s_{1,\hat{N}_n+1}-s'_{1n}| +
|s'_{1n}-s'_{2n}| 
\\ &
+|s'_{2n}-s_{N_1+1,\hat{N}_n+1}| \\
& < 2\epsilon + |s'_{1n}-s'_{2n}| + 2\epsilon \\
\implies 
|\Delta'_{1n2n}| & > (\alpha - 4) \epsilon 
\\ \\
|\Delta'_{1122}| &\ge 0.
\end{align*}
\begin{align*}
2|\Delta'_{1112}||\Delta'_{2212}|+\Delta'_{1112}\Delta'_{2212} & \ge |\Delta'_{1112}||\Delta'_{2212}| \\
& >
[(\alpha-4)\epsilon][(\alpha-4)\epsilon] \\
& > (\alpha-4)^2\epsilon^2
\end{align*}

Plugging in these bounds to the result from Lemma \ref{lemma:1} we get 
\begin{align*}
D_{12} & <
 \frac{\epsilon^2 {\mu'}_1 \mu'_2}{({\mu'}_1 + \mu'_2)^4}
\left[
\left(
2({\mu'}_1+{\mu'}_2)^3
\sum_{i=3}^{N} 
{\mu'}_i (4\epsilon)^2
\right) \right. \\ & \left. +
{{\mu'}_1}^3{\mu'}_2[4(4\epsilon)^2 +
(8\epsilon)^2)] +
{\mu'}_1{\mu'}_2^3[(8\epsilon)^2 +
4(4\epsilon)^2] \right. \\ 
& + 
\left.2((4\epsilon)^2{\mu'}_1^4 +
(4\epsilon)^2{\mu'}_2^{4}) + 
4{\mu'}_1^2{\mu'}_2^2 
[(8\epsilon)^2 + 48\epsilon^2]
\vphantom{\sum_{i=1}^{N}}
\right] \\
& = \frac{\epsilon^2 {\mu'}_1 {\mu'}_2}{({\mu'}_1 + {\mu'}_2)^4} 
\left[ 32({\mu'}_1+{\mu'}_2)^3(1-{\mu'}_1-{\mu'}_2)  
\right. \\ & \left. +
64({\mu'}_1^3{\mu'}_2+{\mu'}_1{\mu'}_3^2) + 32 ({\mu'}_1^4+{\mu'}_2^4) +112{\mu'}_1^2{\mu'}_2^2
\right] \\ 
& \leq
\frac{\epsilon^2 {\mu'}_1 {\mu'}_2}{({\mu'}_1 + {\mu'}_2)^4}
[32({\mu'}_1+{\mu'}_2)^3 - 32({\mu'}_1+{\mu'}_2)^4 
\\ & + 
32({\mu'}_1+{\mu'}_2)^4] \\
& = 
\frac{32 \epsilon^2 {\mu'}_1 {\mu'}_2}{({\mu'}_1 + {\mu'}_2)}
\end{align*}

\begin{align*}
D_{13} > &
 \frac{\epsilon^2 {\mu'}_1 {\mu'}_3}{({\mu'}_1 + {\mu'}_3)^4}
\left[ 
\left(2({\mu'}_1+{\mu'}_3)^3\sum_{i\neq1,3} {\mu'}_n(\alpha-4)^2\right) \right. \\ & \left. + 4(\alpha-4)^2({\mu'}_1^3{\mu'}_3+{\mu'}_1{\mu'}_3^3)
\right. \\ & \left.
+8(\alpha-4)^2({\mu'}_1^4+{\mu'}_3^4) +4(\alpha-4)^2{\mu'}_1^2{\mu'}_3^2] \vphantom{\sum_{n=1}^{N}}
\right] \\
& = 
\frac{\epsilon^2(\alpha-4)^2 {\mu'}_1{\mu'}_3}{({\mu'}_1+{\mu'}_3)^4}
\left[ 
2({\mu'}_1+{\mu'}_3)^3(1-{\mu'}_1-{\mu'}_3)  
\right. \\ & \left. +
4({\mu'}_1^3{\mu'}_3 + {\mu'}_1{\mu'}_3^3) + 8({\mu'}_1^4+{\mu'}_3^4) + 4{\mu'}_1^2{\mu'}_3^2
\right] \\
& >
\frac{\epsilon^2(\alpha-4)^2 {\mu'}_1{\mu'}_3}{({\mu'}_1+{\mu'}_3)^4}
\left[ 
\frac{1}{2}({\mu'}_1+{\mu'}_3)^3(1-{\mu'}_1-{\mu'}_3)  
\right. \\ & \left. +
\frac{1}{2} ({\mu'}_1+{\mu'}_3)^4
\right] \\\
& = 
\frac{\epsilon^2(\alpha-4)^2 {\mu'}_1{\mu'}_3}{({\mu'}_1+{\mu'}_3)^4}
\left[\frac{1}{2}({\mu'}_1+{\mu'}_3)^3 - \frac{1}{2} ({\mu'}_1+{\mu'}_3)^4 
\right. \\ & \left. +
\frac{1}{2} ({\mu'}_1+{\mu'}_3)^4 \right] \\
& = 
\frac{\epsilon^2(\alpha-4)^2 {\mu'}_1{\mu'}_3}{2({\mu'}_1+{\mu'}_3)},
\end{align*}
thus concluding the proof.
\end{proof}

\begin{lemma} \label{lemma:3}
Let $G = (V, S, \mu)$ be a measure network and $G' = (V', S', {\mu'})$ be a coarsening of $G$ satisfying the hypotheses of Lemma~2 with $\epsilon >0$ and $\alpha > 4+4N/\sqrt{N-1}$. 
Then, $\dis(\pi^{ij}) < \dis(\pi^{ik})$.
\end{lemma}

\begin{proof}
Let
$$
G_{12} = \frac{32 \epsilon^2 \mu_1 \mu_2}{(\mu_1 + \mu_2)} \qquad\text{ and } \qquad
G_{13} = \frac{\frac{1}{2}\epsilon^2(\alpha-4)^2 \mu_1\mu_3}{(\mu_1+\mu_3)}.
$$
Then, $\dis^2(\pi^{13}) > \dis^2(\pi^{12})$ if $G_{13} > G_{12}$, or, equivalently, 
$$
(\alpha-4)^2 > 
\frac{32 \epsilon^2 \mu_1 \mu_2}{(\mu_1 + \mu_2)} \times \frac{(\mu_1+\mu_3)}{\frac{1}{2}\epsilon^2\mu_1\mu_3} = 
64\frac{\mu_2}{\mu_3} \frac{\mu_1+\mu_3}{\mu_1+\mu_2}.
$$
We want to find a lower bound that is valid for any choice of $\mu_1, \mu_2, \mu_3$, where $\frac{1}{N} \leq \mu_1,\mu_2,\mu_3 \leq 1-\frac{2}{N}$ and $\mu_1+\mu_2+\mu_3\leq1$. 
This bound is largest when $\mu_1+\mu_2+\mu_3=1$, thus, we let $\mu_2 = 1-\mu_1-\mu_3$.
We thus want to solve 
\begin{align*}    
\max_{\mu_1,\mu_3} G(\mu_1,\mu_3) 
:&= 
64
\frac{1-\mu_1-\mu_3}{\mu_3} \frac{\mu_1+\mu_3}{1-\mu_3} \\
&=
64
\frac{(\mu_1+\mu_3) - (\mu_1+\mu_3)^2}{\mu_3(1-\mu_3)}.
\end{align*}
We first maximize for $\mu_1$:
\begin{align*}    
\frac{\partial}{\partial \mu_1} G(\mu_1, \mu_3) &=
64
\frac{1-2(\mu_1+\mu_3)}{\mu_3(1-\mu_3)} =0 \\
\implies 
\mu_1^* &=
1/2 - \mu_3.
\end{align*}
Note that $\mu_1^*$ is guaranteed to be a maximizer as $G(\mu_1, \mu_3)$ is a negative quadratic in $\mu_1$.
After plugging in $\mu^*_1 = 1/2-\mu_3$ we optimize for $\mu_3$: 
\begin{align}
\frac{\partial}{\partial \mu_3} G(\mu^*_1, \mu_3) &=
\frac{\partial}{\partial \mu_3} 
\left(
64
\frac{(1/2)-(1/2)^2}{\mu_3(1-\mu_3)}
\right) \\ 
&=
\frac{\partial}{\partial \mu_3} 
\left( 
64
\frac{1/4}{\mu_3(1-\mu_3)}
\right) \\
& =-16\frac{1-2\mu_3}{\mu_3^2(1-\mu_3)^2} = 0 \\
\implies\mu_3' &= 1/2
\end{align}
It can be shown that $\mu_3'$ is a minimizer, not a maximizer. 
Therefore, the maximizer $\mu^*_3$ must occur at the boundary of its domain. 
By inspection, we get that $\mu_3^* = 1/N$.
Taken together, we have $\mu^*_1 = 1/2-1/N, \mu^*_2=1/2, \mu^*_3=1/N$, and 
\begin{align*}    
G(\mu_1^*, \mu_2^*, \mu_3^*) 
&= 
64
\frac{\mu_2^*}{\mu_3^*} \frac{\mu_1^*+\mu_3^*}{\mu_1^*+\mu_2^*}  =
64
\frac{1/2}{1/N} \frac{1/2}{1-1/N} = 16\frac{N^2}{N-1}.
\end{align*}
Therefore, $\dis^2(\pi^{13}) > \dis^2(\pi^{12})$ if $\alpha > 4 + 4N/\sqrt{N-1}$.
Since $f(x) = x^2$ is a strictly monotonic bijection on $[0, \infty)$, it follows that $\dis(\pi^{12}) < \dis(\pi^{13})$.
\end{proof}

\section{Weak Isomorphism and Coarsening} \label{sec:weakiso}

As previously mentioned, the space of (compact) measure networks equipped with the Gromov-Wasserstein distance is a pseudometric space. 
A pseudometric space consists of a set $X$ and a pseudo-metric $\hat{d}$, differing from a metric in that $\hat{d}(x,y) = 0$ is possible for $x \neq y$.
Moreover, the Gromov-Wasserstein distance is unique up to weak-isomorphism \cite{Chowdhury_Memoli_2019}, that is $\hat{d}(x,y) = 0$ if and only if $x$ and $y$ are weakly isomorphic.
There are two notions of weak-isomorphism discussed in \cite{Chowdhury_Memoli_2019}, the latter of which is only necessary for infinite measure networks and is therefore beyond the scope of this work.

A pair of measure networks $X = (X, S_X, \mu_X)$ and $Y = (Y, S_Y, \mu_Y)$ is called weakly isomorphic if there exists a third measure network $Z = (Z, S_Z, \mu_Z)$ and injective maps $\phi: Z \to X$ and $\psi: Z \to Y$ such that the following conditions hold:
\begin{enumerate}
\item 
$\phi_*(\mu_Z) = \mu_X \text{ and } \psi_*(\mu_Z) =\mu_Y$

\item 
$\sup_{z_1, z_2 \in Z} \left\lvert \phi^* S_X(z_1,z_2) - \psi^*S_Y(z_1, z_2) \right\rvert= 0$
\end{enumerate}
where $\phi_* = \mu_Z \circ \phi^{-1}$ and $\phi^* w_X(z_1, z_2) = w_Z(\phi(z_1), \phi(z_2))$; $\psi_*$ and $\psi^*$ are defined \textit{mutatis mutandis}.
The concept of a terminal network is discussed in \cite{Chowdhury_Memoli_2019, Chowdhury_Memoli_2023} and is the most concise representation of a measure network in the GW geometry. 
Any measure network in a weak isomorphism class can be represented using its minimal representative via blow-ups \cite{Chowdhury_Needham_Riemann}.
We can determine if a measure network is the minimal representative by checking if there exists a pair of nodes with identical neighborhoods, or equivalently, a pair of nodes that, once merged, induce zero GW distortion.

\begin{figure}[H]
    \centering
    \includegraphics[width=1\linewidth]{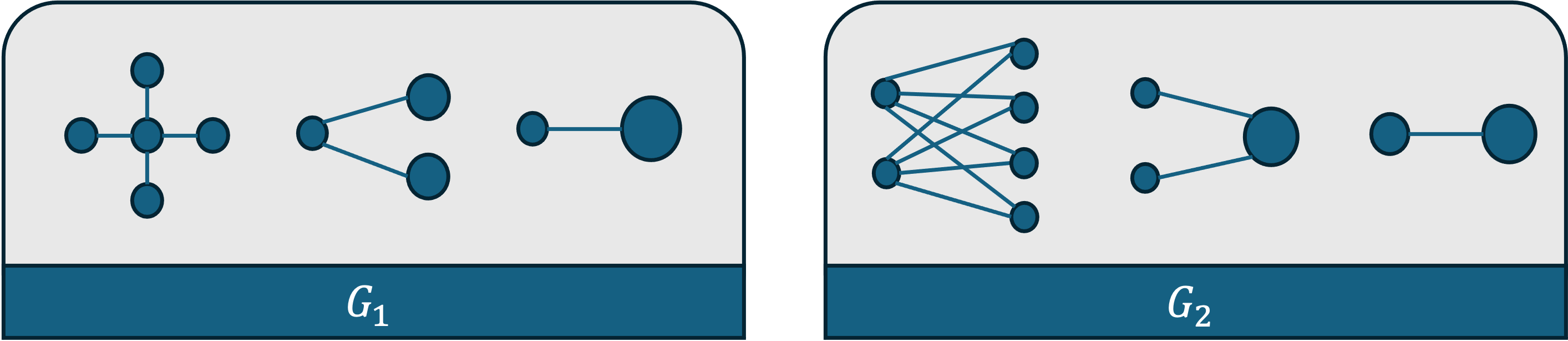}
    \caption{We show here two weak isomorphism classes of graphs. 
    The leftmost networks in each class have uniform mass on nodes and the weights of all edges in each class are equal. 
    The rightmost graphs are minimal representatives or terminal networks in their respective class. 
    Visually, the minimal representatives of the graphs $G_1$ and $G_2$ appear quite similar, but comparing their leftmost representation reveals how different the graphs are.  
    Classes $G_1$ and $G_2$ are both examples of complete bi-partite graphs; these classes of graphs benefit most from coarsening to the minimal representative, as a complete $k$-partite network can be reduced to a $k$-node minimal representative.} 
    \label{fig:placeholder}
\end{figure}

\section{Gromov-Wasserstein Coarsening and Sketching} \label{sec:gw-sketching}

The purpose of this section is to clarify the connection between the graph coarsening problem in the GW setting and the GW measure network sketching problem. 
The GW sketching problem has been previously considered in \cite{memoli_sidiropoulos_singhal_2018} in the case of metric measure spaces (a subset of measure networks \cite{Chowdhury_Memoli_2019}), where notions of duality were established between sketching and clustering with respect to the GW distance (albeit for the GW distance proposed by \cite{sturm2006geometry} which is not computationally feasible). 
The sketching problem in GW space of an $N$-point network $G$ to a $M$-point network $G'$ can be formulated as 
\begin{align}\label{eq:sketch}
\argmin_{G' \in \N_M} d_{GW}^2(G,G') &= 
\argmin_{G' \in \N_M} 
\min_{\pi\in\Pi(\mu,\mu')}
\dis^2(\pi)
\end{align}
We can upper-bound \eqref{eq:sketch} by restricting the feasibility set of the GW distance to those transport plans induced by coarsening matrices, i.e. $\pi = \diag(\mu) C_p$,
\begin{align} \label{eq:upper_bound}
\argmin_{G' \in \N_M} d_{GW}^2(G, G') \leq \argmin_{G'\in\N_M} \min_{C_p\in \C_{N,M}} \dis^2(\diag(\mu) C_p).
\end{align}
As is pointed out in \cite{Chen_2023}, given $C_p$ the $G'$ that minimizes the upper bound in Eq. \eqref{eq:upper_bound} is the semi-relaxed GW barycenter \cite{vincentcuaz_2022}, $\min_{G' \in \N_M} \dgw(G, G')$
\begin{align}
& \leq
\min_{G'\in \N_M} 
\min_{C_p\in \C_{N,M}} 
\dis^2(\diag(\mu)C_p) \\
&= 
\min_{G'\in \N_M} 
\min_{C_p\in \C_{N,M}}
\langle \mathcal{L}^2_2(S,S') \otimes \diag(\mu)C_p, \diag(\mu)C_p \rangle \\
& = 
\min_{C_p \in \C_{N,M}}
\min_{G' \in \N_M} 
\langle \mathcal{L}^2_2(S,S') \otimes \diag(\mu)C_p, \diag(\mu)C_p \rangle \label{eq:pre_best_graph}\\
& = 
\min_{C_p\in\C_{N,M}}
\langle \mathcal{L}^2_2(S,C_w^\top S C_w) \otimes \diag(\mu)C_p, \diag(\mu)C_p \rangle \label{eq:best_graph}
\end{align}
where Eq. \eqref{eq:best_graph} follows from the fact that $S' = C_w^\top S C_w$ is the closed-form solution of the inner minimization problem in Eq. \eqref{eq:pre_best_graph}, as shown in \cite[Appendix B]{Chen_2023}, \cite[Equation 14]{Peyre_Cuturi_Solomon_2016}.
Taken together, we have $\argmin_{G' \in \N_M} d_{GW}^2(G, G')$
\begin{align}
\leq \argmin_{C_p \in \C_{N,M}} 
\lVert 
(S - C_p C_w^\top S C_w C_p^\top)
\odot (\mu \mu^\top)^{1/2}
\rVert_F^2. \label{eq:sketching_vs_coarsening}
\end{align}

\end{document}